\def\year{2021}\relax
\documentclass[letterpaper]{article} 
\usepackage{aaai21}  
\usepackage{times}  
\usepackage{helvet} 
\usepackage{courier}  
\usepackage[hyphens]{url}  
\usepackage{graphicx} 
\usepackage{natbib}  
\usepackage{caption} 
\usepackage{amsmath,amssymb,amsthm,bbm,mathtools}
\usepackage{color}
\usepackage{graphicx,subfigure}
\usepackage{algorithm,algpseudocode}
\usepackage{multirow}
\usepackage{footnote}
\usepackage{enumitem}

\graphicspath{{figs/}}

\newcommand{\real}{\mathbb{R}}

\newcommand{\ones}{\mathbf{1}}

\newcommand{\dd}{\textsc{DD}}
\newcommand{\imdbb}{\textsc{IMDB-B}}
\newcommand{\imdbm}{\textsc{IMDB-M}}
\newcommand{\imdbbi}{\textsc{IMDB-BINARY}}
\newcommand{\imdbmu}{\textsc{IMDB-MULTI}}

\newcommand{\proteins}{\textsc{PROTEINS}}

\newcommand{\ncii}{\textsc{NCI109}}
\newcommand{\mutag}{\textsc{MUTAG}}

\newcommand{\graphsage}{\textsc{GraphSage}}

\newcommand{\setset}{\textsc{Set2Set}}
\newcommand{\sortpool}{\textsc{SortPool}}

\newcommand{\diffpool}{\textsc{DiffPool}}

\newcommand{\gcn}{\textsc{GCN}}
\newcommand{\gunet}{\textsc{Graph U-Net}}

\newcommand{\sagpool}{\textsc{SAGPool}}
\newcommand{\gpool}{\textsc{gPool}}
\newcommand{\gae}{\textsc{GraphAE-Unsupv}}
\newcommand{\otc}{\textsc{OTCoarsening}}
\newcommand{\otcs}{\textsc{OTCoarsening-Sup}}

\DeclareMathOperator{\sigmoid}{sigmoid}

\DeclareMathOperator{\diag}{diag}
\DeclareMathOperator{\gnn}{GNN}

\theoremstyle{definition} 
\theoremstyle{remark}     
\theoremstyle{remark}     
\theoremstyle{plain}      \newtheorem{theorem}{Theorem}
\theoremstyle{plain}      
\theoremstyle{plain}      
\theoremstyle{plain}      \newtheorem{corollary}[theorem]{Corollary}
\theoremstyle{plain}      

\title{Unsupervised Learning of Graph Hierarchical Abstractions \\
with Differentiable Coarsening and Optimal Transport}

\author{%
  Tengfei Ma\thanks{These two authors contribute equally.} \qquad
  Jie Chen\footnotemark[1] \\
}
\affiliations{%
  MIT-IBM Watson AI Lab, IBM Research \\
  Tengfei.Ma1@ibm.com \qquad chenjie@us.ibm.com
}

\begin{document}

\maketitle

\begin{abstract}
  Hierarchical abstractions are a methodology for solving large-scale graph problems in various disciplines. Coarsening is one such approach: it generates a pyramid of graphs whereby the one in the next level is a structural summary of the prior one. With a long history in scientific computing, many coarsening strategies were developed based on mathematically driven heuristics. Recently, resurgent interests exist in deep learning to design hierarchical methods learnable through differentiable parameterization. These approaches are paired with downstream tasks for supervised learning. In practice, however, supervised signals (e.g., labels) are scarce and are often laborious to obtain. In this work, we propose an unsupervised approach, coined \textsc{OTCoarsening}, with the use of optimal transport. Both the coarsening matrix and the transport cost matrix are parameterized, so that an optimal coarsening strategy can be learned and tailored for a given set of graphs without use of labels. We demonstrate that the proposed approach produces meaningful coarse graphs and yields competitive performance compared with supervised methods for graph classification and regression.
\end{abstract}

\section{Introduction}
A proliferation of graph neural networks~\citep{Bruna2014,Henaff2015,Duvenaud2015,Defferrard2016,Kipf2017,Hamilton2017,Chen2018,Velickovic2018,Ying2018a,Liao2019,Xu2019,Scarselli2009,Li2016,Gilmer2017,Jin2017} emerged recently with wide spread applications ranging from theorem proving~\citep{Wang2017}, chemoinformatics~\citep{Jin2017,Fout2017,Schuett2017}, to planning~\cite{Ma2020}. These models learn sophisticated feature representations of a graph and its constituents (i.e., nodes and edges) through layers of feature transformation. Several architectures~\citep{Xu2019,Morris2019,Maron2019} are connected to the Weisfeiler--Lehman (WL) graph isomorphism test~\citep{Shervashidze2011} because of the resemblance in iterative node (re)labeling.

An image analog of graph neural networks is convolutional neural networks, whose key components are convolution and pooling. The pooling operation reduces the spatial dimensions of an image and forms a hierarchical abstraction through successive downsampling. For graphs, a similar hierarchical abstraction is particularly important for maintaining the structural information and deriving a faithful feature representation. A challenge, however, is that unlike image pixels that are spatially regular, graph nodes are irregularly connected and hence pooling is less straightforward.

Several graph neural networks perform pooling in a hierarchical manner. The work of~\citet{Bruna2014} builds a multiresolution hierarchy of the graph with agglomerative clustering, based on $\epsilon$-covering. The work of~\citet{Defferrard2016} and~\citet{Fey2018} employ Graclus that successively coarsens a graph based on the heavy-edge matching heuristic. The work of~\citet{Simonovsky2017} constructs the hierarchy through a combined use of spectral polarity and Kron reduction. These neural networks build the graph hierarchy as preprocessing, which defines in advance how pooling is performed given a graph. No learnable parameters are attached.

Recently, hierarchical abstractions as a learnable neural network module surfaced in graph representation learning. Representative approaches include \diffpool~\citep{Ying2018}, \gunet~\citep{Gao2019}, and~\sagpool~\citep{Lee2019}. All approaches treat the learnable hierarchy as part of the neural network (in conjunction with a predictive model), which is trained with a downstream task in a (semi-)supervised manner.

In practice, however, supervised signals (e.g., labels) are scarce and are often laborious and expensive to obtain. Hence, in this work, we propose an unsupervised approach, called \otc, that produces a hierarchical abstraction of a graph independent of downstream tasks. Therein, node features for the graphs in the hierarchy are derived simultaneously, so that they can be used for different tasks through training separate downstream predictive models. \otc\ consists of two ingredients: a parameterized graph coarsening strategy in the algebraic multigrid (AMG) style; and an optimal transport that minimizes the structural transportation between two consecutive graphs in the hierarchy, thus replacing the cross-entropy or other losses that rely on labeling information. The ``OT'' part of the name comes from Optimal Transport. We show that this unsupervised approach produces meaningful coarse graphs that are structure preserving; and that the learned representations perform competitively with supervised approaches.

The contribution of this work is threefold. First, for unsupervised learning we introduce a new technique based on hierarchical abstraction through minimizing discrepancy along the hierarchy. Second, key to a successful hierarchical abstraction is the coarsening strategy. We develop one motivated by AMG and empirically show that the resulting coarse graphs qualitatively preserve the graph structure. Third, we demonstrate that the proposed technique, combining coarsening and unsupervised learning, performs comparably with supervised approaches but is advantageous in practice facing label scarcity.

\section{Related Work}\label{sec:related}
Hierarchical (a.k.a. multilevel or multiscale) methods are behind the solutions of a variety of problems, particularly for graphs. Therein, coarsening approaches are being constantly developed and applied. Two active areas are graph partitioning and clustering. The former is often used in parallel processing, circuit design, and solutions of linear systems. The latter appears in descriptive data analysis.

Many of the graph hierarchical approaches consist of a coarsening and an uncoarsening phase. The coarsening phase successively reduces the size of a given graph, so that an easy solution can be obtained for the smallest one. Then, the small solution is lifted back to the original graph through successive refinement in the reverse coarsening order. For coarsening, a class of approaches applies heave-edge matching heuristics~\citep{Hendrickson1995,Karypis1998,Dhillon2007}. Loukas and coauthors show that for certain graphs, the principal eigenvalues and eigenspaces of the coarsened and the original graph Laplacians are close under randomized matching~\citep{Loukas2018,Loukas2019}. \citet{BravoHermsdorff2019} show that contracting two nodes into one may be interpreted as perturbing the Laplacian pseudoinverse with an infinitely weighted edge. On the other hand, in the uncoarsening phase, refinement can be done in several ways, including Kernighan-Lin refinement~\citep{Kernighan1970,Shi2000,Luxburg2007} and kernel $k$-means~\citep{Dhillon2007}.

Another class of coarsening approaches selects a subset of nodes from the original graph. Call them coarse nodes; they form the node set of the coarse graph. Other nodes are aggregated with weights to the coarse nodes in certain ways, which, simultaneously define the edges in the coarse graph. Many of these methods were developed akin to algebraic multigrid (AMG)~\citep{Ruge1987}, wherein the coarse nodes, the aggregation rule, and edge weights may be defined based on original edge weights~\citep{Kushnir2006}, diffusion distances~\citep{Livne2012}, or algebraic distances~\citep{Ron2011,Chen2011,Safro2014}. In this work, the selection of the coarse nodes and the aggregation weights are parameterized and learned instead.

Hierarchical graph representation is emerging in deep learning. Representative approaches include \diffpool~\citep{Ying2018}, \gunet~\citep{Gao2019}, and~\sagpool~\citep{Lee2019}. Cast in the above setting, \diffpool\ is similar to the first class of coarsening approaches, whereas \gunet\ and \sagpool\ similar to the latter. All methods are supervised, as opposed to ours.

Our work is additionally drawn upon optimal transport, a tool recently used for defining similarity of graphs~\citep{Vayer2019,Xu2019a}. In the referenced work, Gromov--Wasserstein distances are developed that incorporate both node features and graph structures. Moreover, a transportation distance from the graph to its subgraph is developed by~\citet{Garg2019}. Our approach is based on a relatively simpler Wasserstein distance, whose calculation admits an iterative procedure more friendly to neural network parameterization.

\section{Method}
In this section, we present the proposed method \otc, beginning with two main ingredients: coarsening and optimal transport, followed by a summary of the computational steps in training and the use of the results for downstream tasks.

\subsection{AMG-Style Coarsening}
The first ingredient coarsens a graph $G$ into a smaller one $G_c$. For a differentiable parameterization, an operator will need be defined that transforms the corresponding graph adjacency matrix $A\in\real^{n\times n}$ into $A_c\in\real^{m\times m}$, where $n$ and $m$ are the number of nodes of $G$ and $G_c$ respectively, with $m<n$. We motivate the definition by algebraic multigrid~\citep{Ruge1987}, because of the hierarchical connection and a graph-theoretic interpretation. AMG also happened to be referenced as a potential candidate for pooling in some graph neural network architectures~\citep{Bruna2014,Defferrard2016}.

\subsubsection{Background on Algebraic Multigrid}
AMG belongs to the family of multigrid methods~\citep{Briggs2000} for solving large, sparse linear systems of the form $Ax=b$, where $A$ is the given sparse matrix, $b$ is the right-hand vector, and $x$ is the unknown vector to be solved for. For simplicity, we assume throughout that $A$ is symmetric. The simplest algorithm, two-grid V-cycle, consists of the following steps:
(i) Approximately solve the system with an inexpensive iterative method and obtain an approximate solution $x'$. Let $r=b-Ax'$ be the residual vector.
(ii) Find a tall matrix $S\in\real^{n\times m}$ and solve the smaller residual system $(S^TAS)y=S^Tr$ for the shorter unknown vector $y$.
(iii) Now we have a better approximate solution $x''=x'+Sy$ to the original system. Repeat the above steps until the residual is sufficiently small.

The matrix of the residual system, $S^TAS$, is called the Galerkin coarse-grid operator. One may show that step (ii), if solved exactly, minimizes the energy norm of the error $x-x''$ over all possible corrections from the range of the matrix $S$. Decades of efforts on AMG discover practical definitions of $S$ that both is economic to construct/apply and encourages fast convergence. We depart from these efforts and define/parameterize an $S$ that best suites graph representation learning.

\begin{figure}[t]
  \centering
  \subfigure[Graph $G$]{
    \includegraphics[width=0.28\linewidth]{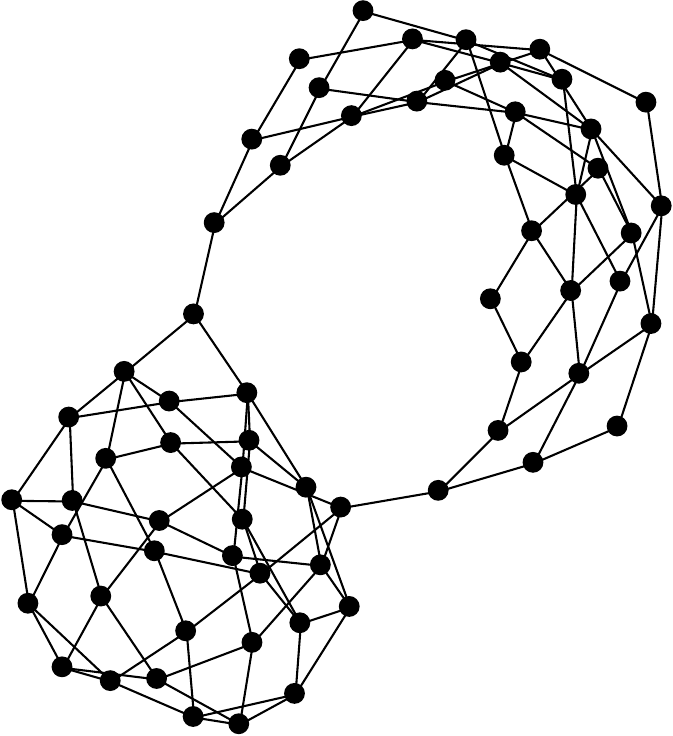}}\quad
  \subfigure[Coarse nodes (red) and fine nodes (blue)]{
    \includegraphics[width=0.28\linewidth]{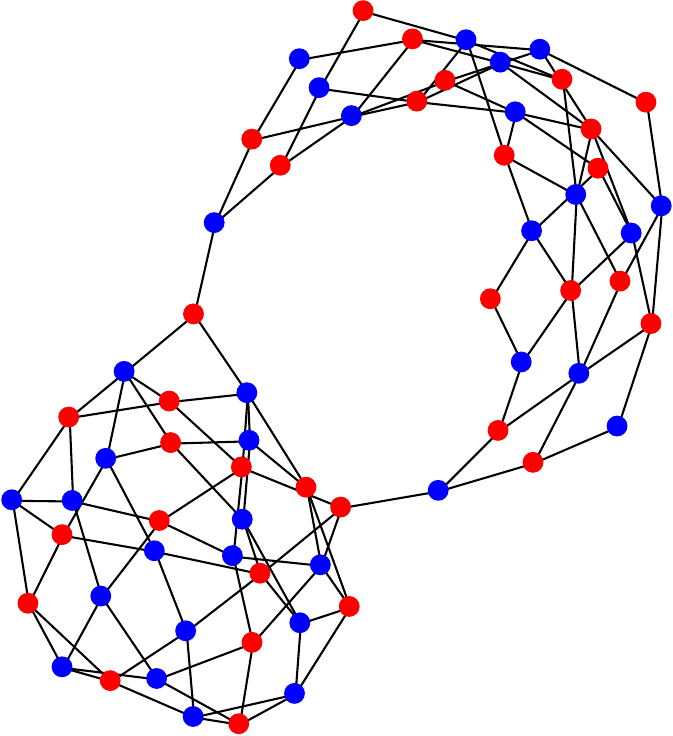}}\quad
  \subfigure[Coarse graph $G_c$ (using only half of the nodes in $G$)]{
    \includegraphics[width=0.28\linewidth]{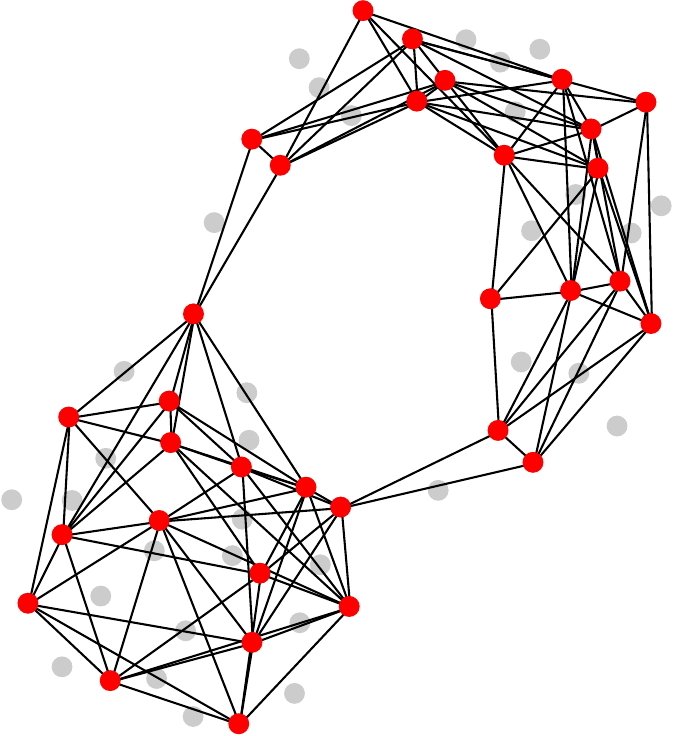}}
  \caption{Example graph and coarsening.}
  \label{fig:interp.aggreg}
\end{figure}

\subsubsection{Coarsening Framework}
Following the above motivation, we settle with the coarsening framework
\begin{equation}\label{eqn:Ac}
  A_c = S^T A S,
\end{equation}
where $S$ is named the \emph{coarsening matrix}. For parameterization, we might have treated $S$ as a parameter matrix, but it requires a fixed size to be learnable and hence it can only be applied to graphs of the same size. This restriction both is unnatural in practice and destroys permutation invariance of the nodes. In what follows, we discuss the properties of $S$ from a graph theoretic view, which leads to a natural parameterization.

\subsubsection{Properties of $S$}
Let $V$ be the node set of the graph $G$. AMG partitions $V$ into two disjoint subsets $C$ and $F$, whose elements are called \emph{coarse nodes} and \emph{fine nodes}, respectively. See Figure~\ref{fig:interp.aggreg}(b). For coarsening, $C$ becomes the node set of the coarse graph and the nodes in $F$ are eliminated.

The rows of the coarsening matrix $S$ correspond to the nodes in $V$ and columns to nodes in $C$. This notion is consistent with definition~\eqref{eqn:Ac}, because the rows and columns of $A_c$ correspond to the coarse nodes. It also distinguishes from \diffpool~\citep{Ying2018}, which although defines the next graph by the same equation~\eqref{eqn:Ac}, does not use the nodes in the original graph as those of the smaller graph.

If $S$ is dense, so is $A_c$. Then, the graphs in the coarsening hierarchy are all complete graphs, which are less desirable. Hence, we would like $S$ to be sparse. Assuming so, one sees that each column of $S$ plays the role of aggregation. For convenience, we define $\chi(j)$ to be the set of nonzero locations of this column and call it the \emph{aggregation set} of the coarse node $j$. The following result characterizes the existence of an edge in the coarse graph.

\begin{theorem}\label{thm:sparse}
  There is an edge connecting two nodes $j$ and $j'$ in the coarse graph if and only if there is an edge connecting the two aggregation sets $\chi(j)$ and $\chi(j')$ in the original graph.
\end{theorem}

\begin{proof}
  We say that the sum of two numbers is \emph{structurally nonzero} if at least one of the numbers is nonzero, even if they sum algebraically to zero (e.g., when one number is the opposite number of the other). Structural nonzero of an element in the adjacency matrix is the necessary and sufficient condition for the existence of the corresponding edge in the graph.

  Recall that $A_c=S^TAS$. For two coarse nodes $j$ and $j'$, one sees that the element $A_c(j,j')$ is structurally nonzero if and only if the submatrix $A(\chi(j),\chi(j'))$ is nonempty. In other words, $j$ and $j'$ are connected by an edge in the coarse graph $G_c$ if and only if there exists an edge connecting $\chi(j)$ and $\chi(j')$ in the original graph $G$. Note that such an edge may be a self loop.
\end{proof}

Hence, in order to encourage sparsity of the coarse graph, many of the aggregation set pairs should not be connected by an edge. One principled approach to ensuring so, is to restrict the aggregation set to contain at most direct neighbors and the node itself. The following corollary is straightforward. We say that the \emph{distance} of two nodes is the number of edges in the shortest path connecting them.

\begin{corollary}\label{cor:sparse}
  If each aggregation set contains at most direct neighbors and the node itself, then there is an edge connecting two nodes in the coarse graph only if the distance between them in the original graph is at most 3.
\end{corollary}

\begin{proof}
  If there is an edge connecting $j$ and $j'$ in the coarse graph, then according to Theorem 1, there is an edge connecting $i\in\chi(j)$ and $i'\in\chi(j')$ in the original graph, for some nodes $i$ and $i'$. Then by the assumption that the elements of $\chi(j)$ are either $j$ or $j$'s direct neighbors and similarly for $\chi(j')$, we know that $j$ and $j'$ are connected by the path $\{j,i,i',j'\}$, which means that the distance between $j$ and $j'$ is at most 3.
\end{proof}

Consequently, in what follows we will let $S$ have the same sparsity structure as the corresponding part of $A+I$. The identity matrix is used to introduce self loops. An illustration of the resulting coarse graph is given in Figure~\ref{fig:interp.aggreg}(c), with self loops omitted.

\subsubsection{Parameterization of $S$}
With the graph-theoretic interpretation of $S$, we now parameterize it. The strategy consists of the following computational steps. First, select coarse nodes in a differentiable manner, so that the sparsity structure of $S$ is determined. Then, compute the nonzero elements of $S$.

The selection of coarse nodes may be done in several ways, such as the top-k approach that orders nodes by projecting their feature vectors along a learnable direction (see, e.g., \citet{Cangea2018,Gao2019}). This approach, however, leverages only node features but not the graph information. To leverage both, we apply one graph convolution
\begin{equation}\label{eqn:alpha}
\alpha=\sigmoid(\widehat{A}XW_{\alpha})
\end{equation}
to compute a vector $\alpha\in\real^{n\times 1}$ that weighs all nodes~\citep{Lee2019}. Here, $\widehat{A}\in\real^{n\times n}$ is the normalized graph adjacency matrix defined in graph convolutional networks~\citep{Kipf2017}, $X\in\real^{n\times d}$ is the node feature matrix, and $W_{\alpha}\in\real^{d\times 1}$ is a parameter vector. The weighting necessitates using sigmoid (or other invertible functions) rather than ReLU as the activation function.

For a coarsening into $m$ nodes, we pick the top $m$ values of $\alpha$ and list them in the sorted order. Denote by $\alpha_s\in\real^{m\times 1}$ such a vector, where the subscript $s$ means sorted and picked. We similarly denote by $\widehat{A}_s\in\real^{n\times m}$ the column-sorted and picked version of $\widehat{A}$.

We let $S$ be an overlay of the graph adjacency matrix with the node weights $\alpha_s$. Specifically, define
\begin{equation}\label{eqn:S}
S=\ell_1\text{-row-normalize}[\widehat{A}_s \odot (\ones \alpha_s^T)],
\end{equation}
where $\ones$ means a column vector of all ones.

There are several reasons why $S$ is so defined. First, $S$ carries the nonzero structure of $\widehat{A}_s$, which, following Corollary~\ref{cor:sparse}, renders more likely a sparse coarse graph. Second, the use of the normalized adjacency matrix introduces self loops, which ensure that an edge in the coarse graph exists if the distance is no more than three, rather than exactly three (which is too restrictive). Third, because both $\widehat{A}_s$ and $\alpha_s$ are nonnegative, the row normalization ensures that the total edge weight of the graph is preserved after coarsening. To see this, note that $\ones^TA_c\ones=\ones^TS^TAS\ones=\ones^TA\ones$.

\subsection{Optimal Transport}
The second ingredient of the proposed \otc\ uses optimal transport for unsupervised learning. Optimal transport~\citep{Peyre2019} is a framework that defines the distance of two probability measures through optimizing over all possible joint distributions of them. If the two measures lie on the same metric space and if the infinitesimal mass transportation cost is a distance metric, then optimal transport is the same as the Wasserstein-1 distance. In our setting, we extend this framework for defining the distance of the original graph $G$ and its coarsened version $G_c$. Then, the distance constitutes the coarsening loss, from which model parameters are learned in an unsupervised manner.

\subsubsection{Optimal Transport Distance}
To extend the definition of optimal transport of two probability measures to that of two graphs, we treat the node features from each graph as atoms of an empirical measure. The coarse node features result from graph neural network mappings, carrying information of both the initial node features and the graph structure. Hence, the empirical measure based on node features characterizes the graph and leads to a natural definition of graph distance.

Specifically, let $M$ be a matrix whose element $M_{ij}$ denotes the transport cost from a node $i$ in $G$ to a node $j$ in $G_c$. We define the distance of two graphs as
\begin{equation}\label{eqn:W.dist}
W_{\gamma}(G,G_c) := \min_{P\in U(a,b)} \langle P,M \rangle - \gamma E(P),
\end{equation}
where $P$, a matrix of the same size as $M$, denotes the joint probability distribution constrained to the space $U(a,b):=\{P\in\real_+^{n\times m} \mid P\ones=a, \,\, P^T\ones=b\}$ characterized by marginals $a$ and $b$; $E$ is the entropic regularization~\citep{Wilson1969}
\[
E(P):=-\sum_{i,j}P_{ij}(\log P_{ij}-1);
\]
and $\gamma>0$ is the regularization magnitude.

Through a simple argument of Lagrange multipliers, it is known that the optimal $P_{\gamma}$ that solves~\eqref{eqn:W.dist} exists and is unique, in the form
$
P_{\gamma} = \diag(u) K \diag(v),
$
where $u$ and $v$ are certain positive vectors of matching dimensions and $K=\exp(-M/\gamma)$ with the exponential being element-wise. The solution $P_{\gamma}$ may be computationally obtained by using Sinkhorn's algorithm~\citep{Sinkhorn1964}: Starting with any positive vector $v^0$, iterate
\begin{multline}\label{eqn:iter}
\text{for $i=0,1,2,\ldots$ until convergence, } \\
u^{i+1} = a \oslash (Kv^i) \text{ and }
v^{i+1} = b \oslash (K^Tu^{i+1}).
\end{multline}
Because the solution $P_{\gamma}$ is part of the loss function to be optimized, we cannot iterate indefinitely. Hence, we instead define a computational solution $P_{\gamma}^k$ by iterating only a finite number $k$ times:
\begin{equation}\label{eqn:Pk}
P_{\gamma}^k := \diag(u^k) K \diag(v^k).
\end{equation}
Accordingly, we arrive at the $k$-step optimal transport distance
\begin{equation}\label{eqn:W.dist.comp}
W_{\gamma}^k(G,G_c) := \langle P_{\gamma}^k,M \rangle - \gamma E(P_{\gamma}^k).
\end{equation}
The distance~\eqref{eqn:W.dist.comp} is the sample loss for training.

\subsubsection{Parameterization of $M$}
With the distance defined, it remains to specify the transport cost matrix $M$. As discussed earlier, we model $M_{ij}$ as the distance between the feature vector of node $i$ from $G$ and that of $j$ from $G_c$. This approach on the one hand is consistent with the Wasserstein distance and on the other hand, carries both node feature and graph structure information.

Denote by $\gnn(A,X)$ a generic graph neural network architecture that takes the graph adjacency matrix $A$ and node feature matrix $X$ as input and produces as output a transformed feature matrix. We produce the feature matrix $X_c$ of the coarse graph through the following encoder-decoder-like architecture:
\begin{equation}\label{eqn:Xc}
Z = \gnn(A,X), \,\,
Z_c = S^T Z, \,\,
X_c = \gnn(A_c,Z_c).
\end{equation}
The encoder produces an embedding matrix $Z_c$ of the coarse graph through a combination of GNN transformation and aggregation $S^T$, whereas the decoder maps $Z_c$ to the original feature space so that the resulting $X_c$ lies in the same metric space as $X$. Then, the transport cost, or the metric distance, $M_{ij}$ is the $p$-th power of the Euclidean distance of the two feature vectors:
\begin{equation}\label{eqn:M}
M_{ij}=\|X(i,:) - X_c(j,:)\|_2^p.
\end{equation}
In this case, the optimal transport distance is the $p$-th root of the Wasserstein-$p$ distance. The power $p$ is normally set as one or two.

\subsection{Training and Downstream Use}
With the technical ingredients developed in the preceding subsections, we summarize the computational steps into Algorithm~\ref{algo:forward}, which is self explanatory.

\begin{algorithm}[ht]
  \caption{Unsupervised training: forward pass}
  \label{algo:forward}
\begin{algorithmic}[1]
  \For{each coarsening level}
  \State Compute coarsening matrix $S$ by~\eqref{eqn:alpha} and~\eqref{eqn:S}
  \State Obtain $A_c$ and $X_c$ by~\eqref{eqn:Ac} and~\eqref{eqn:Xc}
  \State Obtain also node embeddings $Z_c$ from~\eqref{eqn:Xc}
  \State Compute transport cost matrix $M$ by~\eqref{eqn:M}
  \State Compute $k$-step joint probability $P_{\gamma}^k$ by~\eqref{eqn:iter} and~\eqref{eqn:Pk}
  \State Compute current-level loss $W_{\gamma}^k(G,G_c)$ by~\eqref{eqn:W.dist.comp}
  \State Set $G\gets G_c$, $A\gets A_c$, and $X\gets X_c$
  \EndFor
  \State Sum the loss for all coarsening levels as the sample loss
\end{algorithmic}
\end{algorithm}

After training, for each graph $G$ we obtain a coarsening sequence and the corresponding node embedding matrices $Z_c$ for each graph in the sequence. These node embeddings may be used for a downstream task. Take graph classification as an example. For each node embedding matrix, we perform a global pooling (e.g., a concatenation of max pooling and mean pooling) across the nodes and obtain a summary vector. We then concatenate the summary vectors for all coarsening levels to form the feature vector of the graph. An MLP is then built to predict the graph label.

\section{Experiments}
In this section, we conduct a comprehensive set of experiments to evaluate the performance of the proposed method \otc. Through experimentation, we aim at answering the following questions. (i) As an unsupervised hierarchical method, how well does it perform on a downstream task, compared with supervised approaches and unsupervised non-hierarchical approaches? (ii) In a multi-task setting, how well does it perform compared with supervised models trained separately for each task? (iii) Do the coarse graphs carry the structural information of the original graphs (i.e., are they meaningful)?

\subsection{Setup}
We perform experiments with the following data sets: \proteins, \mutag, \ncii, \imdbbi\ (\imdbb\ for short), \imdbmu\ (\imdbm\ for short), and \dd. They are popularly used benchmarks publicly available from~\citet{Kersting2016}. Except \imdbb\ and \imdbm\ which are derived from social networks, the rest of the data sets all come from the bioinformatics domain. Information of the data sets is summarized in Table~\ref{tab:dataset}.

\begin{table}[ht]
  \centering
  \caption{Data sets.}
  \label{tab:dataset}
  \begin{tabular}{lccccccccc}
    \hline
    & \proteins & \mutag & \ncii \\
    \hline
    \# Graphs        & 1,113  & 188   & 4,127 \\
    \# Classes       & 2      & 2     & 2     \\
    Ave. \# nodes    & 39.06  & 17.93 & 29.68 \\
    Ave. node degree & 3.73   & 2.21  & 2.17  \\
    \hline
    \hline
    & \imdbb & \imdbm & \dd \\
    \hline
    \# Graphs        & 1,000 & 1,500 & 1,178  \\
    \# Classes       & 2     & 3     & 2      \\
    Ave. \# nodes    & 19.77 & 13.00 & 284.32 \\
    Ave. node degree & 9.76  & 10.14 & 5.03   \\
    \hline
  \end{tabular}
\end{table}

We gauge the performance of \otc\ with several supervised approaches. They include the plain \gcn~\citep{Kipf2017} followed by a gloabl mean pooling, as well as five more sophisticated pooling methods: 
\sortpool~\citep{Zhang2018}, which retains the top-k nodes for fixed-size convolution;
\diffpool~\citep{Ying2018}, which applies soft clustering;
\setset~\citep{Vinyals2015}, which is used together with \graphsage~\citep{Hamilton2017} as a pooling baseline in~\citet{Ying2018};
\gpool~\citep{Cangea2018,Gao2019}, which retains the top-k nodes for graph coarsening, as is used by \gunet; and
\sagpool~\citep{Lee2019}, which applies self-attention to compute the top-k nodes. Among them, \diffpool, \gpool, and \sagpool\ are hierarchical methods, similar to ours. In addition, we also employ an ablation model, i.e., a supervised version of our coarsening model without using optimal transport distance as the loss function. This model is called \otcs, where we remove the unsupervised learning phase and directly train the coarsening model by using the prediction loss on training data.

Additionally, we take a simple unsupervised baseline. Named \gae, this baseline is a graph autoencoder that does not perform coarsening, but rather, applies GCN twice to respectively encode the node features and decode for reconstruction. The encoder serves the same purpose as that of the plain GCN and the decoder is needed for training without supervised signals. 

\subsection{Experimentation Details}
We evaluate all methods using 10-fold cross validation. For training, we use the Adam optimizer with a tuned initial learning rate and a fixed decay rate $0.5$ for every 50 epochs. We perform unsupervised training for a maximum of 200 epochs and choose the model at the best validation loss. Afterward, we feed the learned representations into a 2-layer MLP and evaluate the graph classification performance.

The weighting vector $\alpha$ (cf. Equation (2)) used for coarse node selection is computed by using 1-layer GCN with activation function $\sigmoid\circ\operatorname{square}$. That is, $\alpha=\sigmoid((\widehat{A}XW_{\alpha})^2)$. The GNNs in Equation (8) for computing the coarse node embeddings $Z_c$ and coarse node features $X_c$ are 1-layer GCNs. The power $p$ in Wasserstein-$p$ (cf. Equation (9)) is fixed as $2$. We use grid search to tune hyperparameters: the learning rate is from $\{0.01, 0.001\}$; and the number of coarsening levels is from $\{1,2,3\}$ for the propoed method and $\{2,3,4\}$ for the compared methods. The coarsening ratio is set to $0.5$ for all methods.

We implement the proposed method and the graph autoencoder by using the PyTorch Geometric library, which is shipped with off-the-shelf implementation of all other compared methods.

The code is available at \url{https://github.com/matenure/OTCoarsening}.

\subsection{Graph Classification}

\begin{table*}[ht]
  \centering
  \caption{Graph classification accuracy (in percentage).}
  \label{tab:results}
  \setlength{\tabcolsep}{0.5em}
  \begin{tabular}{ccccccccc}
    \hline
    Method & \proteins & \mutag & \ncii & \imdbb & \imdbm & \dd \\
    \hline
    \gcn      & 72.3\tiny$\pm$3.1 & 73.4\tiny$\pm$5.2 & 69.6\tiny$\pm$1.3 & 71.3\tiny$\pm$5.2 & 50.5\tiny$\pm$2.4 & 71.8\tiny$\pm$4.1\\
    \setset   & 73.4\tiny$\pm$3.7 & 74.6\tiny$\pm$5.3 & 70.3\tiny$\pm$1.6 & 72.9\tiny$\pm$4.7 & 49.7\tiny$\pm$3.5 & 70.8\tiny$\pm$3.9\\
    \sortpool & 73.5\tiny$\pm$4.5 & 80.1\tiny$\pm$6.7 & 69.1\tiny$\pm$4.5 & 71.6\tiny$\pm$3.6 & 49.9\tiny$\pm$2.1 & 73.7\tiny$\pm$7.7\\
    \diffpool & 74.2\tiny$\pm$3.2 & 84.5\tiny$\pm$7.3 & 71.7\tiny$\pm$2.8 & 74.3\tiny$\pm$3.5 & 50.3\tiny$\pm$2.8 & 73.9\tiny$\pm$3.5\\
    \gpool    & 72.2\tiny$\pm$3.1 & 76.2\tiny$\pm$9.0 & 72.4\tiny$\pm$2.2 & 73.0\tiny$\pm$5.5 & 49.5\tiny$\pm$3.2 & 71.5\tiny$\pm$4.7\\
    \sagpool  & 73.3\tiny$\pm$3.1 & 78.6\tiny$\pm$6.4 & {\bf73.1\tiny$\pm$2.4} & 72.2\tiny$\pm$4.7 & 50.4\tiny$\pm$2.1 & 72.0\tiny$\pm$4.2\\
    \gae      & 74.3\tiny$\pm$3.6 & 84.6\tiny$\pm$8.0 & 66.4\tiny$\pm$4.6 & 72.4\tiny$\pm$5.9 & 49.9\tiny$\pm$2.9 & 76.5\tiny$\pm$2.8\\
    \otcs     & 73.6\tiny$\pm$3.0 & 84.4\tiny$\pm$6.8 & 68.6\tiny$\pm$1.8 & 73.6\tiny$\pm$4.7 & 50.2\tiny$\pm$3.9 & 74.2\tiny$\pm$3.3 \\
    \pmb{\otc}      & {\bf74.9\tiny$\pm$3.9} & {\bf85.6\tiny$\pm$6.2} & 68.5\tiny$\pm$5.2 & {\bf74.6\tiny$\pm$4.9} & {\bf50.9\tiny$\pm$3.3} & {\bf77.2\tiny$\pm$3.1}\\
    \hline
  \end{tabular}
\end{table*}

Graph classification accuracies are reported in Table~\ref{tab:results}. \otc\ outperforms the compared methods in five out of six data sets. Moreover, it improves significantly the accuracy on \dd\ over all supervised baselines. Interestingly, the supervised runner up is almost always \diffpool, outperforming the subsequently proposed \gpool\ and \sagpool. On the other hand, these two methods perform the best on the other data set \ncii, with \sagpool\ taking the first place. On this data set, \otc\ performs on par with the lower end of the compared methods. It appears low-performing, possibly because of the lack of useful node features that play an important role in the optimal transport distance. Our ablation model, \otcs, is comparable to the best performance among all supervised baselines on most datasets, but performs worse than the final unsupervised model \otc\ using optimal transport. 

Based on these observations, we conclude that hierarchical methods indeed are promising for handling graph structured data. Moreover, as an unsupervised method, the proposed \otc\ performs competitively with strong supervised approaches. In fact, even for the simple unsupervised baseline \gae, it outperforms \diffpool\ on \proteins, \mutag, and \dd. This observation indicates that unsupervised approaches are quite competitive, paving the way for possible uses in other tasks.

\subsection{Sensitivity Analysis}
\otc\ introduces parameters owing to the computational nature of optimal transport: (a) the entropic regularization strength $\gamma$; and (b) the number of Sinkhorn steps, $k$. In Figure~\ref{fig:sensitivity}, we perform a sensitivity analysis and investigate the change of classification accuracy as these parameters vary. One sees that most of the curves are relatively flat, except the case of $\gamma$ on \ncii. This observation indicates that the proposed method is relatively robust to the parameters of optimal transport. The curious case of \ncii\ inherits the weak performance priorly observed, possibly caused by the lack of informative input features. 

\begin{figure}[ht]
\centering
\subfigure[Entropic regularization, $\gamma$]{\includegraphics[width=.45\linewidth]{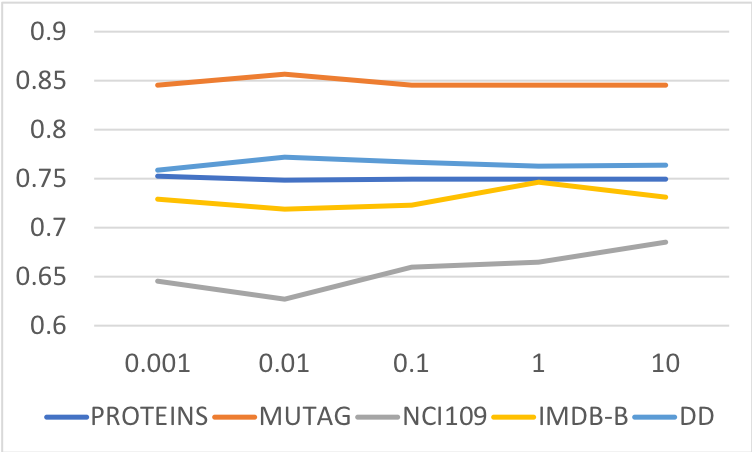}}
\hspace{.5cm}
\subfigure[Sinkhorn steps, $k$]{\includegraphics[width=.45\linewidth]{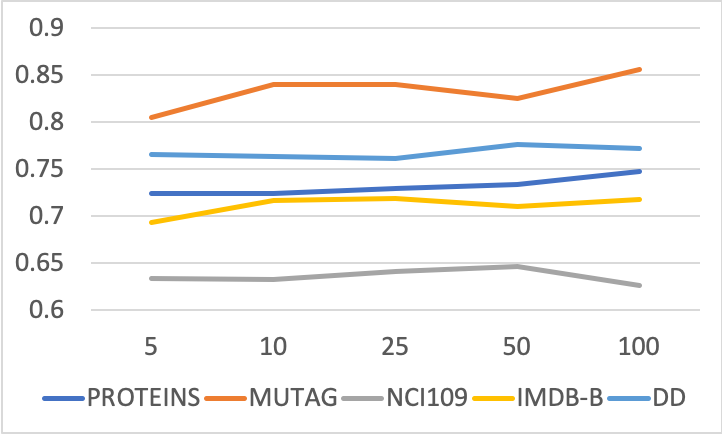}}
\caption{Classification accuracy as parameters vary.}
\label{fig:sensitivity}
\end{figure}

The observation that performance is insensitive to the parameters does not contradict the computational foundation of optimal transport. In the standard use, a transport cost $M$ is given and the optimal plan $P$ is computed accordingly. Hence, $P$ varies with $\gamma$ and $k$. In our case, on the other hand, $M$ is not given. Rather, it is parameterized and the parameterization carries over to the computational solution of $P$. Thus, it is not impossible that the parameterization finds an optimum that renders the loss (Equation (7)) insensitive to $\gamma$ and $k$. In other words, the optimization of Equation (7), with respect to neither $M$ nor $P$ but the parameters therein, turns out to be fairly stable.

\begin{figure*}[ht]
  \centering
    \includegraphics[width=0.47\linewidth]{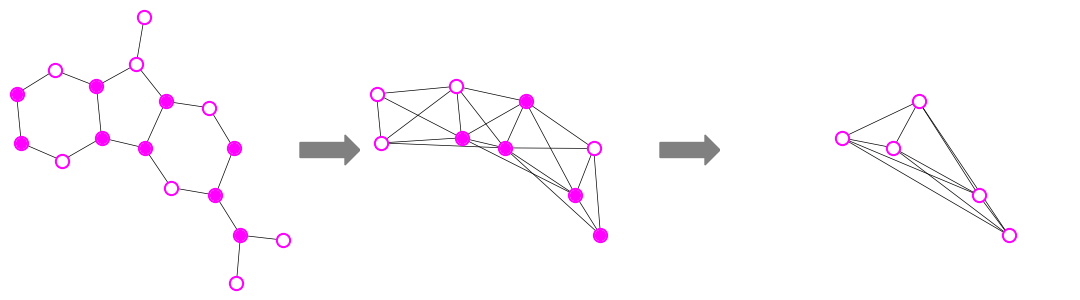}\hspace{0.5cm}
    \includegraphics[width=0.47\linewidth]{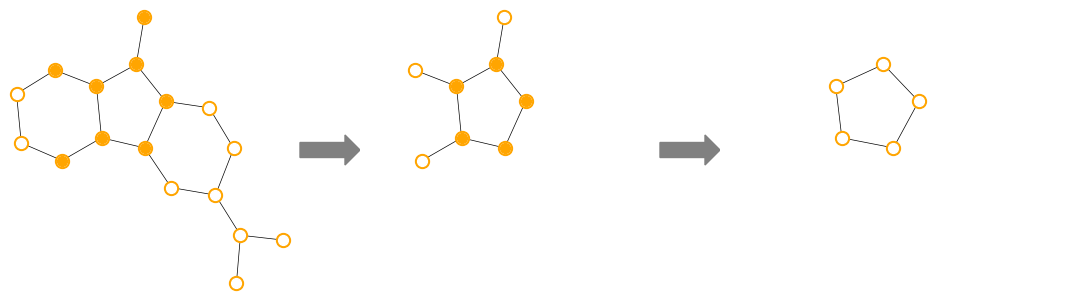}\\
    \includegraphics[width=0.47\linewidth]{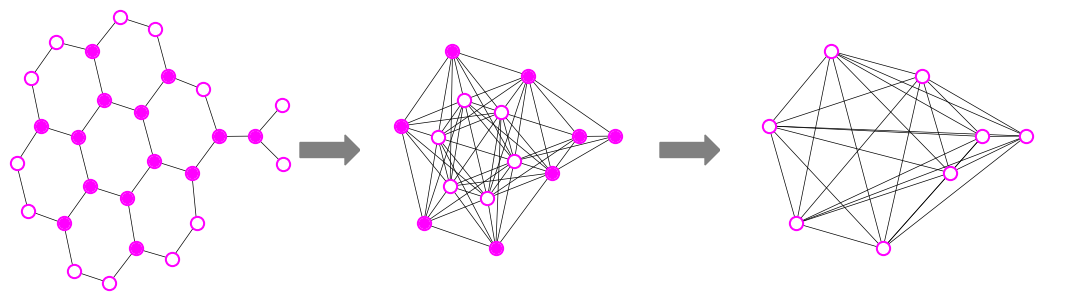}\hspace{0.5cm}
    \includegraphics[width=0.47\linewidth]{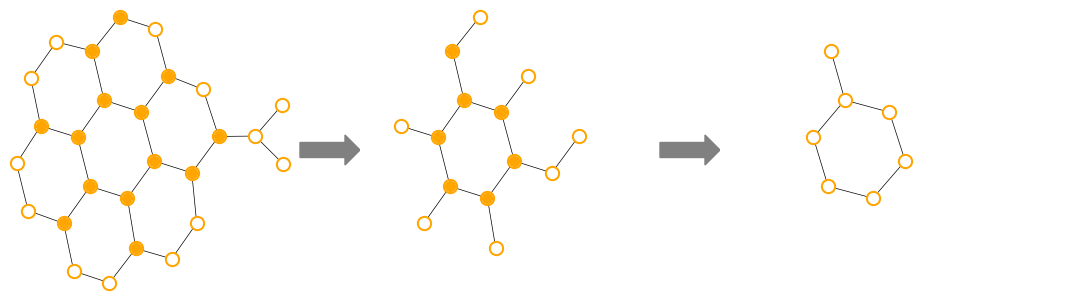}\\
    \includegraphics[width=0.47\linewidth]{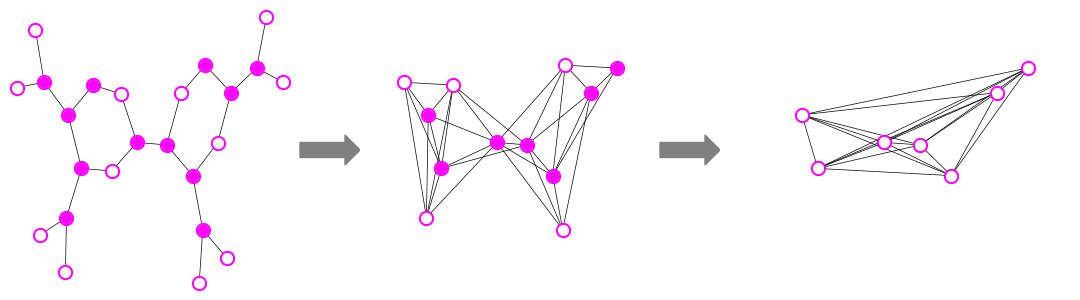}\hspace{0.5cm}
    \includegraphics[width=0.47\linewidth]{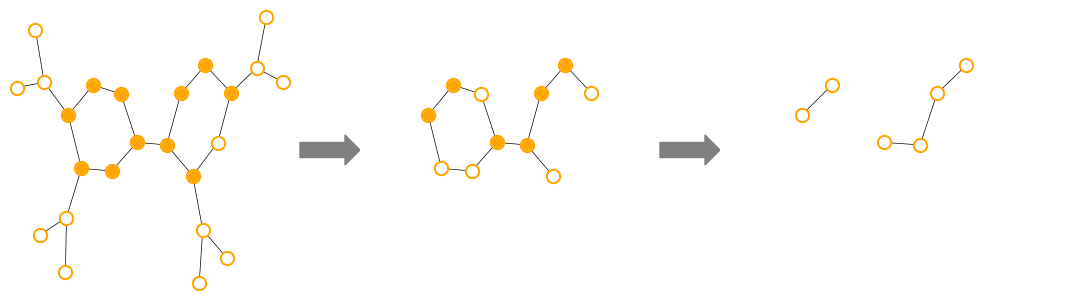}
  \caption{Coarsening sequence for graphs from \mutag. Left (magenta): \otc. Right (orange): \sagpool. Hollow nodes are coarse nodes.}
  \label{fig:coarsening.sequence}
\end{figure*}

\subsection{Multi-Task Learning}
We further investigate the value of unsupervised graph representation through the lens of multi-task learning. We compare three scenarios: (A) a single representation trained without knowledge of the downstream tasks (method: \otc, \gae); (B) a single representation trained jointly with all downstream tasks (methods: \gcn, \setset, \sortpool, \diffpool, \gpool, and \sagpool, all suffixed with ``-joint''); and (C) different representations trained separately with each task (method: \diffpool-sep).

The data set is QM7b~\citep{Wu2018}, which consists of 14 regression targets. Following~\citet{Gilmer2017}, we standardize each target to mean 0 and standard deviation 1; we also use MSE as the training loss but test with MAE. Table~\ref{tab:multitask} reports the MAE and timing results.

\begin{table}[ht]
\centering
\caption{Multi-task regression error and training time (in seconds).}
\label{tab:multitask}
\begin{tabular}{cccc}
\hline
& Method & MAE & Time \\
\hline
(A) & \otc            & 0.6625 & 1296  \\
(A) & \gae            & 0.6749 &  587  \\
(B) & \gcn-joint      & 2.4225 & 2122  \\
(B) & \setset-joint   & 2.4256 & 2657  \\
(B) & \sortpool-joint & 2.4408 & 2652  \\
(B) & \diffpool-joint & 2.4231 & 1100  \\
(B) & \gpool-joint    & 2.4200 & 2117  \\
(B) & \sagpool-joint  & 2.4221 & 1874  \\
(C) & \diffpool-sep   & 0.1714 & 15520 \\
\hline
\end{tabular}
\end{table}

One sees from Table~\ref{tab:multitask} that in terms of regression error, single unsupervised representation (A) significantly outperforms single supervised representations (B), whilist being inferior to separate supervised representations (C). Separate representations outperform single representations at the cost of longer training time, because they need to train 14 separate models whereas others only one. The timings for (B) are comparable with that of (A). The timing variation is caused by several factors, including the architecture difference and dense-versus-sparse implementation. \diffpool\ is implemented with dense matrices, which may be faster compared with other methods that treat the graph adjacency matrix sparse, when the graphs are small.

\subsection{Qualitative Study}
As discussed in the related work section, coarsening approaches may be categorized in two classes: clustering based and node-selection based. Methods in the former class (e.g., \diffpool) coarsen a graph through clustering similar nodes. In graph representation learning, similarity of nodes is measured by not only their graph distance but also the closeness of their feature vectors. Hence, two distant nodes bear a risk of being clustered together if their input features are similar.

On the other hand, methods in the latter class (e.g., \gunet\ and \sagpool) use nodes in the original graph as coarse nodes. If the coarse nodes are connected based on only their graph distance but not feature vectors, the graph structure is more likely to be preserved. Such is the case for \otc, where only nodes within a 3-hop neighborhood are connected. Such is also the case for \gunet\ and \sagpool, where the neighborhood is even more restricted (e.g., only 1-hop neighborhood). However, if two coarse nodes are connected only when there is an edge in the original graph, these approaches bear another risk of resulting in disconnected coarse graphs.

Theoretical analysis is beyond scope. Hence, we conduct a qualitative study and visually inspect the coarsening results. In Figure~\ref{fig:coarsening.sequence}, we show a few graphs from the data set \mutag, placing the coarsening sequence of \otc\ on the left and that of \sagpool\ on the right for comparison. The solid nodes are selected as coarse nodes.

For the graph on the top row, \otc\ selects nodes across the consecutive rings in the first-level coarsening, whereas \sagpool\ selects the ring in the middle. For the graph in the middle row, both \otc\ and \sagpool\ select the periphery of the honeycomb for the first-level coarsening, but differ in the second level in that one selects again the periphery but the other selects the heart. For the graph at the bottom row, \otc\ preserves the butterfly topology through coarsening but the result of \sagpool\ is hard to comprehend.

\section{Conclusion}
Coarsening is a common approach for solving large-scale graph problems in various scientific disciplines. How one effectively selects coarse nodes and aggregates neighbors motivates the present work. Whereas a plethora of coarsening methods were proposed in the past and are used today, these methods either do not have a learning component, or have parameters that need be learned with a downstream task. In this work, we present \otc, which is an unsupervised approach. It follows the concepts of AMG but learns the selection of the coarse nodes and the coarsening matrix through the use of optimal transport. We demonstrate its successful use in graph classification and regression tasks and show that the coarse graphs preserve the structure of the original one. We envision that the proposed idea may be adopted in many other graph learning scenarios and downstream tasks.

\section*{Acknowledgments}
This work is supported in part by DOE Award DE-OE0000910.

\bibliography{reference}


\end{document}